\documentclass{llncs}
\usepackage{amssymb}
\usepackage{amsfonts}
\usepackage{amsmath}
\usepackage[amsmath,thmmarks]{ntheorem}
\usepackage{mathrsfs}
\usepackage[cmtip,arrow]{xy}
\usepackage{pb-diagram,pb-xy}
\usepackage[draft]{fixme}
\usepackage{graphicx}
\usepackage{color}
\usepackage{pstricks}
\usepackage{cite}
\usepackage{subfigure}

\newcommand{\R}{\mathbb{R}}

\newcommand{\des}{\delta}
\newcommand{\eop}{\hfill$\square$}

\newtheorem{prop}[equation]{Proposition}

\theoremstyle{plain}
\theorembodyfont{\normalfont}
\newtheorem{defn}[equation]{Definition}

\usepackage{graphicx}
\usepackage{color}
\usepackage{pstricks}

\begin{document}
\frontmatter          

\title{Geometric tree kernels: Classification of COPD from airway tree geometry
}
\titlerunning{Geometric tree kernels}  
%
\author{Aasa Feragen$^{1,2}$, Jens Petersen$^1$, Dominik Grimm$^2$, Asger Dirksen$^4$, Jesper Holst Pedersen$^5$, Karsten Borgwardt$^{2,3}$ and Marleen de Bruijne$^{1,6}$}
\authorrunning{Feragen, Petersen, Grimm, Dirksen, Pedersen, Borgwardt and de Bruijne} 
%
%
\institute{Department of Computer Science, University of Copenhagen, Denmark$^1$, Max Planck Institute for Intelligent Systems and Max Planck Institute for Developmental Biology, T\"{u}bingen, Germany$^2$, Zentrum f\"{u}r Bioinformatik, Eberhard Karls Universit\"{a}t T\"{u}bingen, Germany$^3$ Lungemedicinsk Afdeling, Gentofte Hospital, Denmark$^4$, Department of Cardiothoracic Surgery, Rigshospitalet, Denmark$^5$, Erasmus MC - University Medical Center Rotterdam, The Netherlands$^6$\\
\email{\{aasa,phup,marleen\}@diku.dk,\\ \{aasa.feragen, dominik.grimm, karsten.borgwardt\}@tuebingen.mpg.de},\\ WWW home page: 
\texttt{http://www.image.diku.dk/aasa}
}

\maketitle              

\begin{abstract}
{\it Methodological contributions:} This paper introduces a family of kernels for analyzing (anatomical) trees endowed with vector valued measurements made along the tree. While state-of-the-art graph and tree kernels use combinatorial tree/graph structure with discrete node and edge labels, the kernels presented in this paper can include geometric information such as branch shape, branch radius or other vector valued properties. In addition to being flexible in their ability to model different types of attributes, the presented kernels are computationally efficient and some of them can easily be computed for large datasets ($N \sim 10.000$) of trees with $30-600$ branches. Combining the kernels with standard machine learning tools enables us to analyze the relation between disease and anatomical tree structure and geometry. {\it Experimental results:} The kernels are used to compare airway trees segmented from low-dose CT, endowed with branch shape descriptors and airway wall area percentage measurements made along the tree. Using kernelized hypothesis testing we show that the geometric airway trees are significantly differently distributed in patients with Chronic Obstructive Pulmonary Disease (COPD) than in healthy individuals. The geometric tree kernels also give a significant increase in the classification accuracy of COPD from geometric tree structure endowed with airway wall thickness measurements in comparison with state-of-the-art methods, giving further insight into the relationship between airway wall thickness and COPD. {\it Software:} Software for computing kernels and statistical tests is available at \url{http://image.diku.dk/aasa/software.php}.
\end{abstract}

\section{Introduction}

Anatomical trees like blood vessels, dendrites or airways, carry information about the organs they are part of, and if we can meaningfully compare anatomical trees and measurements made along them, then we can learn more about aspects of disease related to the anatomical trees~\cite{hasegawa,washko}. For example, airway wall thickness is known to be a biomarker for Chronic Obstructive Pulmonary Disease (COPD), and in order to compare airway wall thickness measurements in different patients, a typical approach is to compare average airway wall area percentage measurements for given airway tree generations of particular subtrees~\cite{hasegawa,hackx}. These approaches assume that measurements made in different locations of the lung are comparable on a common scale, which is not always the case~\cite{hasegawa,hackx}. If we can compare tree structures attributed with measurements made along them in a way which respects the structure and geometry of the tree, then we can be more robustly compare measurements whose values are location sensitive. In this paper we present a family of kernels for comparing anatomical trees endowed with vector attributes, and use these to get a more detailed understanding of how COPD correlates with airway structure and geometry.

\vspace{2mm}

\noindent {\bf Related work.} Several approaches to statistics on attributed (geometric) trees have recently appeared, and some of them were applied to airway trees~\cite{feragen_iccv11,feragen_miccai,nye,sturmmean}. These methods only consider branch length or shape and do not allow for using additional measurements along the airway tree, such as branch radius or airway wall area percentage. Moreover, these methods are computationally expensive~\cite{feragen_sspr}, or need a set of leaf labels~\cite{nye,feragen_miccai}, making them less applicable for general trees. S\o{}rensen~\cite{lauge} treats the airway tree as a set of attributed branches which are matched and then compared using a dissimilarity embedding  combined with a $k$-NN classifier. The matching introduces an additional computational cost and makes the approach vulnerable to incorrect matches.


\emph{Kernels} are a family of similarity measures equivalent to inner products between data points implicitly embedded in a Hilbert space. Kernels are typically designed to be computationally fast while discriminative for a given problem, and often give nonlinear similarity measures in the original data space. Using the Hilbert space, many Euclidean data analysis methods are extended to kernels, such as classification~\cite{libsvm} or hypothesis testing~\cite{gretton}. Kernels are popular because they give computational speed, modeling flexibility and access to linear data analysis tools for data with nonlinear behavior.

There are kernels available for structured data such as strings~\cite{rational,leslie}, trees~\cite{vis_smola}, graphs~\cite{randomwalk,weisfeiler_lehman,karsten_icdm} and point clouds~\cite{bach}. The current state-of-the-art graph kernel in terms of scalability is the Weisfeiler-Lehman (WL)~\cite{weisfeiler_lehman} kernel, which compares graphs by counting isomorphic labeled subtrees of a particular type and ``radius'' $h$. The WL scales linearly in $h$ and the number of edges, but the scalability depends on algorithmic constructions for finite node label sets. Thus, the WL kernel, like most fast kernels developed in natural language processing and bioinformatics~\cite{leslie,vis_smola,rational}, does not generalize to vector-valued branch attributes.

Walk- and path based kernels~\cite{randomwalk,bach,karsten_icdm}, which reduce to comparing sub-walks or -paths of the graphs, are state-of-the-art among kernels which include continuous-valued graph attributes. Random walk-type kernels~\cite{randomwalk,bach} suffer from several problems including tottering~\cite{mahe} and high computational cost. The shortest path kernel~\cite{karsten_icdm} by default only considers path length, and some of the kernels developed in this paper can be viewed as extensions of the shortest path kernel.

\vspace{2mm}

\noindent {\bf Contributions.} We develop a family of kernels which are computationally fast enough to run on large datasets, and can incorporate any vectorial attributes on nodes\footnote{Our formulation allows both node and edge attributes, as edge attributes are equivalent to node attributes on rooted trees: assign each edge attribute to its child node.}, e.g., shape or airway wall measurements. Using the kernels in classification and hypothesis testing experiments, we show that classification of COPD can be substantially improved by taking geometry into account. This illustrates, in particular, that airway wall area percentage measurements made at different locations in the airway tree are not comparable on a common scale. 

We compare the developed kernels to state-of-the-art methods. We see, in particular, that COPD can also be detected from combinatorial airway tree structure using state-of-the-art kernels on tree structure alone, but we show that these contain no more information than a branch count kernel, as opposed to the geometric tree kernels.

\section{Geometric trees and geometric tree kernels}

Anatomical trees like airways are \emph{geometric trees}: they consist of both combinatorial tree structure and branch geometry (e.g., branch length or shape), where continuous changes in the branch geometry can lead to continuous transitions in the combinatorial tree structure. In addition to its geometric embedding, a geometric tree can be adorned with additional features measured along the tree, e.g., airway branch radius, airway wall thickness, airway wall thickness/branch radius, airway wall area percentage in an airway cross section, etc. 

\begin{defn}
A \emph{geometric tree} is a pair $(T, x)$ where $T = (V, E, r)$ is a combinatorial tree with nodes $V$, root $r$ and edges $E \subset V \times V$, and $x \colon V \to \R^n$ is an assignment of (geometric) attributes from a vector space $\R^n$ to the nodes of $T$, e.g. $3D$ position or landmark points. An \emph{attributed geometric tree} is a triple $(T, x, a)$ where $(T, x)$ is a geometric tree and $a \colon x(T) \to \R^d$ is a map assigning a vector valued attribute $a(p) \in \R^d$ to each point $p \in x(T)$.
\end{defn}

A common strategy for defining kernels on structured data such as trees, graphs or strings is based on combining kernels on sub-structures such as strings, walks, paths, subtrees or subgraphs~\cite{rational,leslie,vis_smola,randomwalk,weisfeiler_lehman,karsten_icdm,bach}. These are all instances of the so-called \emph{R-convolution kernels} by Haussler~\cite{haussler}. We shall use \emph{paths} in trees as building blocks for defining kernels on trees.

Let $(T, x)$ be a geometric tree. Given vertices $v_i, v_j \in V$ there is a unique path $\pi_{ij}$ from $v_i$ to $v_j$ in the tree, defined by the sequence of visited nodes:
\[
\pi_{ij} = \left[v_i, p^{(1)}(v_i), p^{(2)}(v_i), \ldots, w, \ldots, p^{(2)}(v_j), p^{(1)}(v_j), v_j\right],
\]
where $p^{(0)}(v) = v$, $p^{(1)}(v) = p(v)$ is the parent node of $v$, more generally $p^{(k)}(v) = p(p^{(k-1)}(v))$, and $w$ is the highest level common ancestor of $v_i$ and $v_j$ in $T$. We call $\pi_{ij}$ the \emph{node-path} from $v_i$ to $v_j$ in $T$ and for each $j$ let the \emph{node-rootpath} $\pi_{jr}$ be  the node-path from $v_j$ to the root.

If the geometric node attributes $x(v) \colon I \to \R^n$ denote embeddings of the edge $(v, p(v))$ into the ambient space $\R^n$, a continuous path $x_{ij} \colon [0, 1] \to \R^n$ can be defined, connecting the embedded nodes $x(v_i), x(v_j) \in \R^n$ along the embedded tree $x(V) \subset \R^n$. We call $x_{ij}$ the \emph{embedded path} from $x_i$ to $x_j$ in $T$.

Throughout the rest of this section, we shall define different kernels for pairs of trees $T_1$ and $T_2$, where $T_i = (V_i, E_i, r_i, x_i, a_i)$ are attributed geometric trees (including non-attributed geometric trees as a special case with $a_i \equiv 1$), $i = 1, 2$. All kernels defined in this section are positive semidefinite, as they are sums of linear and Gaussian kernels composed with known feature maps. This is a necessary condition for a kernel to be equivalent to an inner product in a Hilbert space~\cite{bishop}, needed for the analysis methods used in Sec.~\ref{experiment}.

\subsection{Path-based tree kernels}

\noindent {\bf All-pairs path kernels.} The all-pairs path kernel is a basic path-based tree kernel. Given two geometric trees, it is defined as
\begin{equation} \label{allpaths_kernel}
K_{a} \left( T_1, T_2 \right) = \sum_{ \begin{array}{c}(v_i, v_j) \in V_1 \times V_1,\\ (v_k, v_l) \in V_2 \times V_2\end{array}} k_p(p_{ij}, p_{kl}),
\end{equation}
where $k_p$ is a kernel defined on paths, and $p_{ij}$, $p_{kl}$ are paths connecting $v_i$ to $v_j$ and $v_k$ to $v_l$ in $T_1$ and $T_2$, respectively -- for instance, $\pi_{ij}$ and $\pi_{kl}$, or $x_{ij}$ and $x_{kl}$, as defined above. Note that if the path kernel $k_p$ is a path length kernel, then the all-pairs path kernel is a special case of the shortest path kernel on graphs~\cite{karsten_icdm}.

The kernel $k_p$ should take large values on paths that are geometrically similar, and small values on paths which are not, giving a measure of the alignment of the two tree-paths $p_{ij}$ and $p_{kl}$, making $K_a$ an overall assessment of the similarity between the two geometric trees $T_1$ and $T_2$. The all-pairs path kernel is nice in the sense that it takes every possible choice of paths in the trees into account. It is, however, expensive: The computational cost is $\mathcal{O}(|V|^4) \cdot \mathcal{O}(k_p)$, where $|V| = \max \{ |V_1|, |V_2| \}$ and $\mathcal{O}(k_p)$ is the cost of the path kernel $k_p$. 

\vspace{2mm}

\noindent {\bf Rootpath kernels.} The computational complexity can be reduced by only considering rootpaths, giving a rootpath kernel $K_r$ defined as:
\begin{equation} \label{rootpath_kernel}
K_r \left( T_1, T_2 \right) = \sum_{v_i \in V_1, v_j \in V_2} k_p(p_{ir}, p_{jr})
\end{equation}
where $k_p$ is a path kernel as before, and $p_{ir}$ is the path from $v_i$ to the root $r$. This reduces the computational complexity to $\mathcal{O}(|V|^2)  \mathcal{O}(k_p)$.

\subsection{Path kernels}

The modeling capabilities and computational complexity of the kernels $K_a$ and $K_r$ depend on the choices of path kernel $k_p$ and path representation $p$. 

\vspace{2mm} 
\noindent {\bf Landmark point representation of embedded paths.} From a shape modeling point of view, equidistantly sampled landmark points give a reasonable representation of a path through the tree. Representing paths by $N$ equidistantly sampled landmark points $x_{ij} \in (\R^n)^N$, the path kernel $k_p = k_x$ is either a linear or Gaussian kernel:
\begin{equation} \label{geo_pathwise_kernel}
k_x(x_{ij}, x'_{kl}) = 
\left\{
\begin{array}{ll}
\langle x_{ij} , x'_{kl} \rangle & \textrm{(linear, i.e., dot product)}\\
e^{-\lambda\|x_{ij} - x'_{kl}\|^2_2} & \textrm{(Gaussian)}
\end{array} \right.
\end{equation}
for a scaling parameter $\lambda$ which regulates the width of the Gaussian. 

\vspace{2mm}

\noindent {\bf Node-path kernels.} The landmark point kernels are expensive to compute (see Table~\ref{computational_complexities}). In particular, two embedded tree-paths may have large overlapping segments without having a single overlapping equidistantly sampled landmark point, as the distance between landmark points depends on the length of the entire path. Thus, most landmark points will only appear in one path, giving little opportunity for recursive algorithms or dynamic programming that take advantage of repetitive structure. To enable such approaches, we use node-paths. 

Assume that $\pi^1 = [\pi^1(1), \pi^1(2), \ldots, \pi^1(m)]$ and $\pi^2 = [\pi^2(1), \pi^2(2), \ldots, \pi^2(l)]$ are node-paths in $T_1, T_2$, respectively, as defined above, that is, sequences of consecutive nodes $\pi^i(j) \in V_i$ in $T_i$. We define
\begin{equation} \label{nodepath}
k_\pi(\pi^1, \pi^2) = \left\{ 
\begin{array}{ll}
\sum_{i = 1}^L k_n\left( x_1(\pi^1(i)), x_2(\pi^2(i)) \right) &  \textrm{ if } |\pi^1| = |\pi^2| = L\\
0 & \textrm{ otherwise}
\end{array}
\right.
\end{equation}
where $k_n$ is a node kernel. In this paper $k_n(v_1, v_2)$ is either a linear kernel without/with additional attributes $a_i$:
\[
\langle x_1(v_1) , x_2(v_2) \rangle, \quad \langle x_1(v_1) , x_2(v_2)\rangle \langle a_1(v_1) , a_2(v_2) \rangle
\]
or a Gaussian kernel with/without attributes $a_i$
\[
e^{- \lambda_1 \|x_1(v_1) - x_2(v_2)\|^2}, \quad e^{- \lambda_1\|x_1(v_1) - x_2(v_2)\|^2} \cdot e^{- \lambda_2 \|a_1(v_1) - a_2(v_2) \|^2},
\]
where the Gaussian weight parameters are heuristically set to the inverse dimension of the relevant feature space, i.e., $\lambda_1 = \frac{1}{n}$ and $\lambda_2 = \frac{1}{d}$.

Now the node-rootpath tree kernel $K_r$ can be rewritten as: 
\begin{equation} \label{node_decomp_1}
K_r(T_1, T_2) = \sum_{l=1}^h \sum_{v_1 \in V_1^l} \sum_{v_2 \in V_2^l} \sum_{i=1}^l k_n \left( x_1(\pi_1(i)), x_2(\pi_2(i)) \right),
\end{equation}
where $h = \min \{\textrm{height}(T_i), i = 1,2\}$. This can be reformulated as a weighted sum of node kernels, giving substantially improved computational complexity:

\begin{prop} \label{desdecomp}
For each $l \le h$, let $V_i^l$ be the set of vertices at level $l$ in $T_i$. Then
\begin{equation} \label{des_decomp}
K_r(T_1, T_2) = \sum_{i = 1}^h \sum_{v_1 \in V_1^l} \sum_{v_2 \in V_2^l} \langle \des_{v_1} , \des_{v_2} \rangle k_n \left(v_1, v_2\right),
\end{equation}
where $\des_{v_i}$ is an $h$-dimensional vector whose $j^{\textrm{th}}$ coefficient counts the number of descendants of $v_i$ at level $j$ in $T_i$, respectively. The complexity of computing $K_r$ is $\mathcal{O}(h \max_l |V^l|^2  (n + h))$. 

When $k_n$ is a linear kernel $\langle x_1(v_1), x_2(v_2) \rangle$ or $\langle x_1(v_1), x_2(v_2) \rangle  \langle a_1(v_1), a_2(v_2) \rangle$, the kernel $K_r$ can be further decomposed as
\begin{equation}
\begin{array}{c}
K_r(T_1, T_2) = \sum_{l = 1}^h \langle \gamma(T_1, l), \gamma(T_2, l)\rangle,\\
\gamma (T_i, l) = \left\{
\begin{array}{ll}
\sum_{v \in V_i^l} x_i(v) \otimes \des(v), & \textrm{\emph{no} } a_i\\
\sum_{v \in V_i^l} a_i(v) \otimes x_i(v) \otimes \des(v), & \textrm{\emph{with} } a_i
\end{array} \right.
\end{array}
\end{equation}
at total complexity  $\mathcal{O} (|V|  h  n)$ / $\mathcal{O} (|V|  h  n  d)$ (without/with attributes $a_i$). Here, $\otimes$ denotes the Kronecker product.
\end{prop}

\begin{proof}
Eq.~\eqref{des_decomp} follows from the fact that the terms $k_n(v_1, v_2)$ in kernel~\eqref{node_decomp_1} will be counted once for every pair $(w_1, w_2)$ of descendants of $v_1$ and $v_2$, respectively, which are at the same level. The descendant vectors $\des(v_i)$ for all $v_i \in V_i$ can be precomputed using dynamical programming at computational cost $\mathcal{O}(|V|  h)$, since $\des(v) = [1, \oplus_{p(w) = v} \des(w)]$, where $\oplus$ is defined as left aligned addition of vectors\footnote{e.g., $[a, b, c] \oplus [d, e] = [a  + d, b + e, c]$.}. The cost of computing $K_r$ is thus
\[
\mathcal{O}(|V|h + h\max_l |V^l|^2(h+n)) = \mathcal{O}(h \max_l |V^l|^2(n+h)),
\]
where $\mathcal{O}(n)$ is the cost of computing each node kernel $k_n(v_1, v_2)$. 

To prove~\eqref{des_decomp} without attributes $a_i$, let $x_1(v_1), x_2(v_2), \des(v_1)$ and $\des(v_2)$ be column vectors and use the Kronecker trick ($\langle a_i, a_j\rangle \langle b_i, b_j\rangle = \langle b_i \otimes a_i, b_j \otimes a_j \rangle$):
\[\begin{array}{ll}
K(T_1, T_2) & = \sum_{l = 1}^h \sum_{v_1 \in V_1^l} \sum_{v_2 \in V_2^l} \langle \des(v_2), \des(v_1)\rangle \langle x_1(v_1),  x_2(v_2) \rangle\\ 
 & = \sum_{l = 1}^h \langle \sum_{v_2 \in V_2^l} \left( x_2(v_2) \otimes \des(v_2) \right), \sum_{v_1 \in V_1^l} x_1(v_1) \otimes \des(v_1)\rangle\\
 & = \sum_{l = 1}^h \langle \gamma(T_1, l), \gamma(T_2, l) \rangle.
\end{array}
\]
The total complexity is thus $\mathcal{O} (\max_{l, i} |V_i^l|  h  n) + \mathcal{O}(|V|  h) =  \mathcal{O} (|V|  h  n)$. Similar analysis proves the attributed case.
\qed \eop
\end{proof}

\vspace{-0.3cm}

\subsection{Pointcloud kernels} \label{sec:pointcloud}

Anatomical measurements can also be weighed by location using $3D$ position alone in a pointcloud kernel. The pointcloud kernel does not use the tree structure but treats each edges in the tree as a point and compares all points:
\begin{equation} \label{pointcloud_global}
K_{PC}(T_1, T_2) = \sum_{e_1 \in E_1} \sum_{e_2 \in E_2} k_e(e_1, e_2)
\end{equation}
where $k_e$ is a kernel on attributed edges. We use a Gaussian edge kernel (GPC):
\begin{equation}
k_e(e_1, e_2) = \underbrace{e^{- \lambda_1 \left\|x(e_1) - x(e_2) \right\|^2}}_{c_1} \underbrace{e^{- \lambda_2 \left\|a(e_1) - a(e_2) \right\|^2}}_{c_2}.
\end{equation}
The kernel is designed to weight the contribution to the total kernel $K_{PC}$ of the airway wall area percentage kernel value $c_1$ between edges $e_1$ and $e_2$ by the geometric alignment of the same edges, defined by the geometric kernel value $c_2$.

\begin{table}[b]
\centering
\begin{tabular}{|c|c|c|}
\hline
 & Embedded paths & Node-path\\
 & ($m$ landmark points)  & \\
 \hline
All-paths & $\mathcal{O}(|V|^4  m  n)$ & $\mathcal{O}(|V|^2  h  \max_l |V^l|^2(n + h)$\\
\hline
Root-paths & $\mathcal{O} (|V|^2  m  n)$ & $\mathcal{O}(h  \max_l |V^l|^2)  (n + h)$\\
\hline
Attributed all-paths  & N/A & $\mathcal{O}(h  |V|^2  \max_l |V^l|^2)  (n + d + h)$ \\
\hline
Attributed root-paths & N/A & $\mathcal{O}(h  \max_l |V^l|^2)  (n + d + h)$ \\
\hline
Attributed linear root-paths & N/A &  $\mathcal{O} (|V|  h  n  d)$\\
\hline
Pointcloud kernel & N/A & $\mathcal{O}(|V|^2  n  d)$\\
\hline
\end{tabular}
\caption{Computational complexities for the considered kernels. Trees are assumed to be embedded in $\R^n$ and admit additional vector valued measurements in $\R^d$.}
\label{computational_complexities}
\end{table}

\subsection{Baseline kernels}

The kernels presented in this paper are compared to a set of baseline kernels. Standard airway wall area percentage measurements are often compared by using an average measure over parts of the tree or a vector of average measures in chosen generations. We use two baseline airway wall area percentage kernels:
\begin{equation} \label{eq:aaw}
K_{AAW\%}(T_1, T_2) = e^{- \|\hat{a}_1 - \hat{a}_2 \|^2},
\end{equation}
\begin{equation} \label{eq:agaw}
K_{AgAW\%}(T_1, T_2) = e^{- \|(\hat{a}_1)_{(3-6)} - (\hat{a}_2)_{(3-6)} \|^2}
\end{equation}
where $\hat{a}_i$ is the average airway wall area percentage averaged over all centerline points in the tree, and $(\hat{a}_i)_{(3-6)}$ is a $4$-dimensional vector of average airway wall area percentages averaged over all centerline points in generations $3-6$ in tree $T_i$. For these kernels (AAW\%, AgAW\%), linear versions were also computed (i.e. $e^{-\|w_1 - w_2\|^2}$ replaced with $\langle w_1, w_2 \rangle$), but the corresponding classification results are not reported as they were consistently weaker than the Gaussian kernels.

Airway segmentation is likely more difficult in diseased as opposed to healthy subjects, as also observed by~\cite{lauge}. In order to check whether the number of detected branches may be a bias in the studied kernels, we compare our kernels to a linear and a Gaussian branchcount kernel (LBC/GBC) defined by
\begin{equation} \label{eq:lbc}
K_{LBC}(T_1, T_2) = \sharp(V_1) \cdot \sharp(V_2), \quad K_{GBC}(T_1, T_2) = e^{-\| \sharp(V_1) - \sharp(V_2) \|^2}.
\end{equation}
The linear kernel LBC is the most natural, since the Hilbert space associated to a linear kernel on $w \in \R^n$ is just $\R^n$. However, a linear kernel on $1$-dimensional input cannot be normalized, as \eqref{normalize_kernel} produces a kernel matrix with entries $\equiv 1$, and the GBC kernel is used for comparison in Table~\ref{kernelres} to show that the geometric tree kernels are, indeed, measuring something other than branch count.

Several state-of-the-art graph kernels were also used. The random walk kernel~\cite{randomwalk} did not finish computing within reasonable time. The shortest path kernel~\cite{karsten_icdm} was computed with edge number as path length, and the Weisfeiler-Lehman kernel~\cite{weisfeiler_lehman} was computed with node degree as node label. Results are reported in Tables~\ref{comp_time},~\ref{perm_test} and~\ref{kernelres}.

\section{Experiments} \label{experiment}

\vspace{-2mm}

Analysis was performed on airway trees segmented from CT-scans of $1966$ subjects from a national lung cancer screening trial. Triangulated mesh representations of the interior and exterior wall surface were found using an optimal surface based approach \cite{petersen2011}, and centerlines were extracted from the interior surface using front propagation \cite{exact}. As the resulting centerlines are disconnected at bifurcation points, the end points were connected using a shortest path search within an inverted distance map of the interior surface. The airway centerline trees were normalized using person height as an isotropic scaling parameter. Airway wall thickness and airway radius were estimated from the shortest distance from each surface mesh vertex to the centerline. The measurements were grouped and averaged along the centerline by each nearest landmark point. 

Out of the $1966$ participants, $980$ were diagnosed with COPD level 1-3 based on spirometry, and $986$ were symptom free. The minimal/maximal/average number of branches in an airway tree was $29/651/221.5$, respectively.


\begin{table}[t]
\centering
\begin{tabular}{|c|c|c|c|c|c|c|c|c|}
\hline
Kernel & Linear & Gaussian & average & average & Shortest & Weisfeiler \\
 & root-node- & branchcount & AW \% & generation & path & Lehman\\
 & path & & & AW \% & &  ($h = 10$)\\
\hline
Comp.~time & $46$ m $43$ s & $23$ m $3$ s & $0.87$ s & $ 1.61$ s & $42$ m $26$ s & $59$ m $23$ s \\
\hline
\end{tabular}
\caption{Runtime for selected kernels on a larger set of $9710$ airway trees.}
\label{comp_time}
\end{table}

\subsection{Kernel computation and computational time} \label{kernelcomp}

The kernels listed in table~\ref{kernelres} were implemented in Matlab~\footnote{Software: \url{http://image.diku.dk/aasa/software.php}; published software was used for SP, WL~\cite{weisfeiler_lehman}.} and computed on a 2.40GHz Intel Core i7-2760QM CPU with 32 GB RAM. Each kernel matrix was normalized to account for difference in tree size:
\begin{equation} \label{normalize_kernel}
K_{\text{norm}}(T_1, T_2) = \frac{K(T_1, T_2)}{\sqrt{K(T_1, T_1) K(T_2, T_2)}}.
\end{equation}
An exception was made for linear kernels between scalars (LBC and AAW\%), since normalization such kernels results gives matrix coefficients $ \equiv 1$.

Computation times for the different kernels used in the classification experiments in Section~\ref{classification} on $1966$ airway trees are shown in Table~\ref{kernelres}. To demonstrate scalability, some of the kernels were ran on $9710$ airways from a longitudinal study of the $1966$ participants, see Table~\ref{comp_time}. The slower kernels were not included. 

For classification and hypothesis testing, a set of $1966$ airway trees from $1966$ distinct subjects was used ($980$ diagnosed with COPD at scan time).

\subsection{Hypothesis testing: Two-sample test for means} \label{hypothesis}

Let $\mathcal{X}$ denote a set of data objects. Given any positive semidefinite kernel $k \colon \mathcal{X} \times \mathcal{X}$ there exists an implicitly defined feature map $\phi \colon \mathcal{X} \to \mathcal{H}$ into a reproducing kernel Hilbert space $(\mathcal{H}, \langle \cdot \rangle )$ such that $k(x_1, x_2) = \langle \phi(x_1), \phi(x_2) \rangle$ for all $x_1, x_2 \in \mathcal{X}$~\cite{bishop}. Hypothesis tests can be defined in $\mathcal{H}$ to check whether two samples $A, B \subset \mathcal{X}$ are implicitly embedded by $\phi$ into distributions on $\mathcal{H}$ that have, e.g., the same means $\mu_A = \mu_B$~\cite{gretton}. Denote by $\hat{\mu}_A$ and $\hat{\mu}_B$ the sample means of $\phi(A)$ and $\phi(B)$ in $\mathcal{H}$, respectively; we use as a test statistic the distance
\[
T(A, B) = \|\hat{\mu}_A - \hat{\mu}_B\|_\mathcal{H}
\]
between the sample means and check the null hypothesis using a permutation test. Writing $|A| = a$ and  $|B| = b$, we divide $\mathcal{X} = A \cup B$ into $N$ random partitions $A_i, B_i$ of size $|A_i| = a$ and $|B_i| = b$, $i = 1 \ldots N$, compute the test statistic $T_i$ for each partition, and compare it with the statistic $T_0$ obtained for the original partition $\mathcal{X} = A \cup B$. An approximate $p$-value giving the probability of $\phi(A)$ and $\phi(B)$ coming from distributions with identical means $\mu_A = \mu_B$ is now given by $p = \frac{|\{T_i | T_i \ge T_0, i = 1 \ldots N\}| + 1}{N + 1}$. The $T$ statistic can be computed from a kernel matrix since distances in $\mathcal{H}$ can be derived directly from the values of $k(\mathcal{X}, \mathcal{X})$ using the binomial formula:
\[
\begin{array}{ll}
\|\hat{\mu}_A - \hat{\mu}_B\|^2 & = \langle \frac{1}{a} \sum_{i =1}^a \phi(a_i) - \frac{1}{b} \sum_{j=1}^b \phi(b_j), \frac{1}{a} \sum_{i =1}^a \phi(a_i) - \frac{1}{b} \sum_{j=1}^b \phi(b_j) \rangle\\
 & = \frac{1}{a^2} \sum_{i = 1}^a \sum_{m = 1}^a \langle \phi(a_i), \phi(a_m) \rangle - \frac{2}{ab} \sum_{i = 1}^a \sum_{j = 1}^b \langle \phi(a_i), \phi(b_j) \rangle\\
 & + \frac{1}{b^2} \sum_{j = 1}^b \sum_{n = 1}^b  \langle \phi(b_j), \phi(b_jn) \rangle\\
 & = \frac{1}{a^2} \sum_{i = 1}^a \sum_{m = 1}^a k\left(a_i, a_m\right) - \frac{2}{ab} \sum_{i = 1}^a \sum_{j = 1}^b k\left( a_i, b_j \right)\\
 & + \frac{1}{b^2} \sum_{j = 1}^b \sum_{n = 1}^b  k\left(b_j, b_n \right).
\end{array}
\]

Using the test with selected kernels we show that healthy airways and COPD airways do not come from the same distributions (Table~\ref{perm_test}).

\begin{table}[b]
\centering
\begin{tabular}{|c|c|c|c|c|c|c|c|c|c|}
\hline
Kernel & Gaussian & Gaussian & Average  & Generation-\\
 & pointcloud & branchcount & AW-wall \% & average AW-wall \% \\
\hline
$p$-value & $9.99 \cdot 10^{-5}$ & $9.99 \cdot 10^{-5}$ & $9.99 \cdot 10^{-5}$ & $9.99 \cdot 10^{-5}$\\
\hline
\hline
Kernel & Linear & Linear & Shortest & Weisfeiler\\
 & all-node-path & Root-node-path & path & Lehman\\
\hline
$p$-value & $9.99 \cdot 10^{-5}$ & $9.99 \cdot 10^{-5}$ & $9.99 \cdot 10^{-5}$ & $9.99 \cdot 10^{-5}$ \\
\hline
\end{tabular}
\caption{Permutation tests for the means of the COPD patient and healthy subject samples. All permutation tests are made with $10.000$ permutations.}
\label{perm_test}
\end{table}

\subsection{COPD classification experiments} \label{classification}

\begin{table}[t]
\centering
\begin{tabular}{|c|c|c|c|}
\hline
\bf Kernel type & \bf Mean class. & \bf Kernel matrix & \bf Mean class.\\
 & \bf accuracy & \bf computation time & \bf accuracy\\
 & & &  $\bf K + K_{GBC}$\\
\hline
Rootpath, linear~\eqref{rootpath_kernel},~\eqref{geo_pathwise_kernel} &  $62.4 \pm 0.7 \%$ & $9$ h $9$ m $20$ s & $\bf 66.8 \pm 0.4 \%$\\
\hline
Rootpath, Gaussian~\eqref{rootpath_kernel},~\eqref{geo_pathwise_kernel} &  $\bf64.9 \pm 0.4 \%$ & $6$ h $53$ m $21$ s & $\bf 68.2 \pm 0.5 \%$ \\
\hline
All-node-paths, linear~\eqref{allpaths_kernel},~\eqref{nodepath} & $62.0 \pm 0.6 \%$ & $3$ h $7$ s & $63.2 \pm 0.5 \%$\\
\hline
Root-node-path, linear~\eqref{rootpath_kernel},~\eqref{nodepath} & $61.8 \pm 0.7 \%$ & $4$ m $24$ s & $62.9 \pm 0.8 \%$\\
\hline
Root-node-path, Gaussian~\eqref{rootpath_kernel},~\eqref{nodepath} &  $\bf 64.4 \pm 0.8 \%$ & $97$ h $21$ m $45$ s & $\bf 64.9 \pm 0.6 \%$\\
\hline
Root-node-path, linear, $a_i$~\eqref{rootpath_kernel},~\eqref{nodepath} & $58.6 \pm 0.6 \%$ & $19$ m $44$ s & $62.3 \pm 0.8 \%$\\
airway wall area \% attribute & & & \\
\hline
Pointcloud, Gaussian~\eqref{pointcloud_global} & $\bf 64.4 \pm 0.6 \%$ & $18$ h $40$ m $26$ s & $\bf 66.5 \pm 0.6 \%$\\
\hline
Branchcount, linear~\eqref{eq:lbc} & $62.3 \pm 1.0 \%$ & $0.08$ s & N/A \\
\hline
Branchcount, Gaussian & $63.3 \pm 0.4 \%$ & $0.2$ & N/A \\
\hline
Linear kernel on $\%$ & $56.2 \pm 0.6 \%$ & $0.62$ s & $63.3 \pm 0.5 \%$\\
average airway wall area~\eqref{eq:aaw} & & & \\
\hline
Gaussian kernel on average & $60.3 \pm 0.2 \%$ & $0.35$ s & $63.3 \pm 0.5 \%$\\
airway wall area \%, & & & \\
generations $3-6$~\eqref{eq:agaw} & & & \\
\hline
Shortest path~\cite{karsten_icdm} & $62.6 \pm 0.4 \%$ & $20$ m $24$ s & $63.4 \pm 0.4 \%$\\
\hline
Weisfeiler Lehman ($h = 10$)~\cite{weisfeiler_lehman} & $62.1 \pm 0.5 \%$ & $14$ m $40$ s & $62.9 \pm 0.5 \%$\\
\hline
\end{tabular}
\caption{Classification results for COPD on $1966$ individuals, of which $893$ have COPD.}
\vspace{-0.5cm}
\label{kernelres}
\end{table}
Based on the kernel matrices corresponding to the kernels described in Sec.~\ref{kernelcomp} for a set of $1966$ airway trees, classification into COPD/healthy was done using a support vector machine (SVM)~\cite{libsvm}. The SVM slack parameter was trained using cross validation on $90 \%$ of the entire dataset, and tested on the remaining $10\%$. This experiment was repeated $10$ times and the mean accuracies along with their standard deviations are reported in Table~\ref{kernelres}. All kernel matrices were combined with the GBC kernel matrix in order to check whether the kernels were, in fact, detecting something other than branch number.

\section{Discussion}

We have constructed a family of kernels that operate on geometric trees, and seen that they give a fast way to compare large sets of trees. We have applied the kernels to hypothesis testing and classification of COPD based on airway tree structure and geometry, along with state-of-the-art methods. We show that there is a connection between COPD and airway wall area percentage, and the COPD detected based on our weighted airway wall area percentage kernels is stronger than what can be found using average airway wall area percentage measurements over different airway tree generations, which is commonly done~\cite{hasegawa,hackx}.

Efficient kernels for trees with vector-valued node attributes are difficult to design because algorithmically, similarity of vector-valued attributes is more challenging to efficiently quantify than equality of discrete-valued attributes. Nevertheless, some of the defined kernels for vector-attributed trees are fast enough to be applied to large datasets from clinical trials. 

Vector-valued attributes are important from a modeling point of view, as they allow inclusion of geometric information such as branch shape or clinical measurements in the trees. However, there is a tradeoff between computational speed and optimal use of the attributes. The efficient node paths are less robust than the embedded paths in airway segmentations with missing or spurious branches, and we observe a small drop in classification performance in Table~\ref{kernelres}. Rootpath kernels are introduced to improve computational speed. However, they do introduce a bias towards increased weighting of parts of the tree close to the root, which are contained in more root-paths. Gaussian local kernels perform significantly better than linear ones (Table~\ref{kernelres}), which is particularly pronounced in the pointcloud kernel. In convolution kernels based on quantification of substructure similarity rather than isomorphic substructure, all the dissimilar substructures are still contributing to the total value of the kernel, and the Gaussian local kernel downscales the effect of dissimilar substructures much more efficiently than the linear kernel. This is particularly pronounced in kernels that use geometric weighting of airway wall measurement comparison. Unfortunately, however, algorithmic constructions like the Kronecker trick (Prop.~\ref{desdecomp}) do not work for the Gaussian kernels, which do not scale well to larger datasets.

Using hypothesis tests for kernels we show that the healthy and COPD diagnosed airway trees come from different distributions. Using SVM classification we show that COPD can be detected by kernels that depend on tree geometry, tree geometry attributed with airway wall area percentage measurements, or combinatorial airway tree structure. Another efficient detector of COPD is the number of branches detected in the airway segmentation. It is thus important to clarify that our defined kernels are not just sophisticated ways of counting the detected branches. Combining the GBC kernel with the other kernels improves classification performance of the geometrically informed tree and pointcloud kernels, showing that these kernels must necessarily contain independent information, and the connection between COPD and airway shape is more than differences in detected airway branch numbers. In contrast, graph kernels that only use the tree structure are not significantly improved by combination with the branch count kernel. Future work includes efficient ways of computing all-paths kernels with linear node attributes, efficient kernels for trees with errors in them, as well replacing the Gaussian local kernels with more efficient RBF type kernels.

\section*{Acknowledgements}

This research was supported by the Danish Council for Independent Research $|$ Technology and Production Sciences; the Lundbeck Foundation; AstraZeneca; The Danish Council for Strategic Research; Netherlands Organisation for Scientific Research; and the DFG project ''Kernels for Large, Labeled Graphs (LaLa)''.

\bibliographystyle{plain}

\end{document}